\documentclass{article}
\usepackage{iclr2026_conference,times}

\usepackage{amsmath,amsfonts,bm}

\def\eqref#1{equation~\ref{#1}}

\def\1{\bm{1}}

\DeclareMathAlphabet{\mathsfit}{\encodingdefault}{\sfdefault}{m}{sl}
\SetMathAlphabet{\mathsfit}{bold}{\encodingdefault}{\sfdefault}{bx}{n}

\usepackage{amssymb}
\usepackage{amsthm}
\usepackage{booktabs}
\usepackage{float}
\usepackage{graphicx}
\usepackage{placeins}
\usepackage{hyperref}
\usepackage{microtype}
\usepackage{url}

\newtheorem{definition}{Definition}

\newtheorem{theorem}{Theorem}

\newtheorem{corollary}{Corollary}
\newtheorem{proposition}{Proposition}

\newcommand{\Zspace}{\mathcal{Z}}
\newcommand{\Mspace}{\mathcal{M}}
\newcommand{\Vset}{\mathcal{V}}
\newcommand{\Relset}{\mathcal{R}}
\newcommand{\Edges}{\mathcal{E}}
\newcommand{\Tspace}{\mathfrak{T}}
\newcommand{\Pathset}{\mathcal{P}}
\newcommand{\Alignset}{\mathfrak{A}}
\newcommand{\dM}{d_{\Mspace}}

\newcommand{\Gstar}{G^{\star}}
\newcommand{\GR}{G_R}
\newcommand{\Cand}{\Gamma}
\newcommand{\Feas}{F_R}
\newcommand{\SupportRegion}{C_R}
\newcommand{\TV}{\operatorname{TV}}
\newcommand{\osc}{\operatorname{osc}}
\newcommand{\dist}{\operatorname{dist}}
\newcommand{\supp}{\operatorname{supp}}

\title{Grounding LLM Reasoning under Incomplete Graph Evidence}
\author{Jiaqi Li\textsuperscript{1}, Fanghui Song\textsuperscript{2}\\
1. Tianjin Normal University, College of Computer and Information Engineering,
Tianjin, China\\
2. Harbin Institute of Technology, School of Mathematics, Harbin, China\\
jqli@tjnu.edu.cn, fanghuis1997@gmail.com}

\iclrfinalcopy
\begin{document}
\maketitle
\setlength{\abovedisplayskip}{6pt plus 2pt minus 2pt}
\setlength{\belowdisplayskip}{6pt plus 2pt minus 2pt}

\begin{abstract}
Knowledge graphs can guide large language models (LLMs) reasoning, but the graph seen by a system is usually a retrieved, linked, temporally scoped, and incomplete evidence state rather than a complete account of truth.  
We develop a theoretical perspective on grounding observable LLM trajectories under such incomplete graph evidence.  
The evidence state induces entity anchors, typed relation residuals, path energies, and support regions, while the language model supplies a prior over candidate trajectories.  We show that, under open-world incompleteness, no hard rule based only on the observed state can both reject every false
unsupported trajectory and retain every true-but-unobserved one.  
We then characterize soft grounding as a KL-regularized deformation of the LLM prior: finite slack preserves support for unsupported but non-contradicted trajectories, whereas hard conditioning appears as an infinite-penalty limit.  
The framework also yields stability bounds under evidence perturbations and clarifies the constraint regimes appropriate for GraphRAG, KGQA, graph agents, constrained decoding, and faithful generation.  
The claims are evidence-relative: KG compatibility is treated as declared support, not factual truth.
\end{abstract}


\section{Introduction}
\label{sec:introduction}

Large language models provide a flexible interface for knowledge-intensive
reasoning, but final-answer accuracy can hide a fragile evidential route.
Consider the Complex WebQuestions question \citep{talmor2018web}: ``Lou Seal is
the mascot for the team that last won the World Series when?''  A supported route
anchors \textsc{Lou Seal}, follows the mascot relation to the \textsc{San
Francisco Giants}, and then follows the team's most recent championship relation
to the \textsc{2014 World Series}.  A model can nevertheless reach 2014 after
drifting through the wrong team, reversing a typed relation, or inventing an
unsupported bridge.  Conversely, a correct route can be removed when retrieval
omits the final championship edge.  Answer-level evaluation collapses these
failures even though they have different causes and remedies.

We study the observable or reconstructed trajectory that a system displays,
retrieves, scores, revises, or asks a user to trust.  This object is an
output-facing sequence of claims rather than a window into the model's hidden
causal computation.  It is nonetheless the right unit for many KG-enhanced
systems, because their evidence is exposed through linked entities, retrieved
subgraphs, typed paths, graph actions, or provenance.

The central complication is epistemic.  The KG available to a system is rarely a
complete truth set.  It is an evidence state $R$ shaped by retrieval, entity
linking, schema matching, temporal snapshots, and provenance filtering.  A
low-energy path in $R$ is evidence support, not a certificate of truth; failure
to find such a path is missing support, not automatically falsity.  Grounding is
therefore a relation relative to an observed state, not an oracle evaluation
against an inaccessible complete graph.

This distinction yields an elementary but consequential impossibility.  Two
latent worlds can produce exactly the same observed state while disagreeing
about a missing claim.  Any hard rule that uses only the observed evidence must
make the same decision in both worlds.  It cannot simultaneously retain every
true-but-unobserved trajectory and reject every false unsupported trajectory.
The limitation is informational, not computational: stronger search or a better
embedding cannot recover a distinction erased by the observation process.

We develop a theoretical perspective in which KG-enhanced grounding is modeled
as posterior deformation over candidate trajectories.  The language model
supplies a prior over trajectories, while the observed evidence state induces
energy through entity anchors, typed relation residuals, path composition,
reliability costs, and support regions.  Grounding then reweights the prior
toward lower-energy trajectories while preserving a principled role for prior
uncertainty and graph incompleteness.  Finite slack keeps unsupported but
non-contradicted trajectories in consideration; hard feasible-set conditioning
appears as an infinite-penalty boundary case.

Geometry supplies the structure of the evidence energy.  Entities act as
anchors, typed relations as transformations or compatibility maps, paths as
relational compositions, and energy sublevel sets as evidence-relative support
regions.  The construction is metric-space based and does not require a
particular Euclidean KG embedding.

Our main contributions are as follows.
\begin{itemize}
    \item We formalize KG-enhanced LLM reasoning as evidence-relative grounding over observable trajectories, explicitly accounting for incomplete graph evidence.
    \item We present a relational-geometric framework that represents KG evidence through entity anchors, typed relation compatibility, path composition, energy-based support, and evidence-relative support regions.
    \item We prove that hard grounding is information-theoretically limited under incomplete evidence: the observed graph alone cannot, in general, distinguish false unsupported trajectories from true but unobserved ones.
    \item We characterize soft grounding as KL-regularized posterior deformation and derive its Gibbs form, hard-limit behavior, evidence-margin conditions, and stability guarantees under bounded evidence perturbations.
\end{itemize}

We position the paper as a theoretical perspective.  Its goal is to clarify the form and limits of evidence-relative grounding, rather than to introduce a new KG encoder, decoding algorithm, or empirical state-of-the-art result.

\begin{figure}[t]
\centering
\IfFileExists{Figure/Figure1.png}{\includegraphics[width=0.85\linewidth]{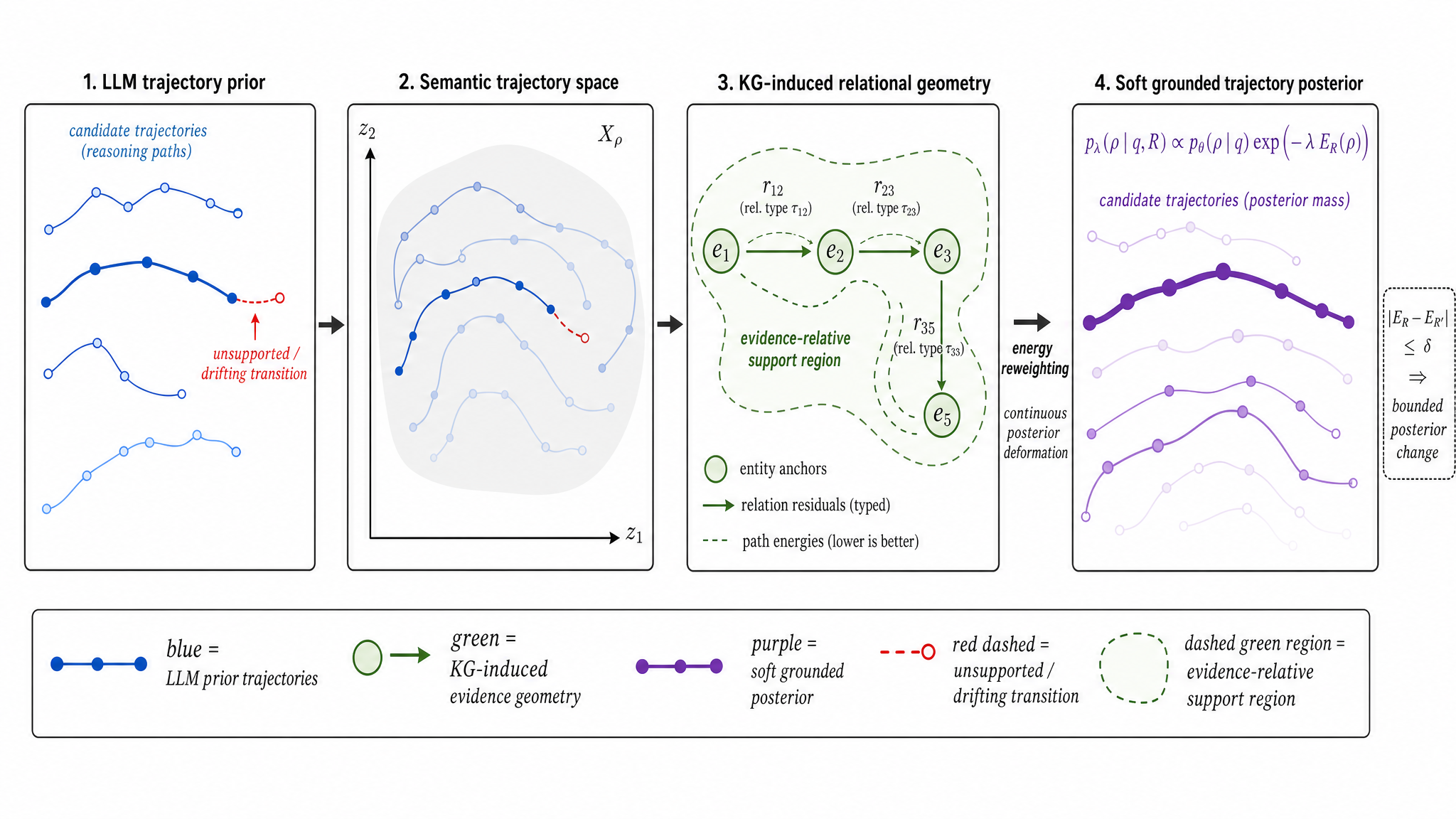}}{%
\fbox{\parbox{0.94\linewidth}{\centering LLM prior, random evidence state,
KG-induced energy, and KL-regularized grounded posterior.}}}
\caption{Evidence-relative geometric grounding.  The observed state $R$ induces
typed relational energies that deform, rather than replace, the LLM trajectory
prior.  Hard feasibility is an infinite-penalty boundary case.}
\label{fig:geometric-grounding}
\end{figure}

\section{Background and Critical Positioning}
\label{sec:background}

Several strands of KG-enhanced LLM research already expose intermediate objects
that can be inspected: chain-of-thought, tree-structured reasoning,
graph-of-thought procedures, and reason-and-act prompting make textual states or
actions visible \citep{wei2022CoT,Yao2023,yao2023react,besta2024got}.  RAG
conditions generation on retrieved evidence \citep{Lewis2020rag}, while GraphRAG
and graph neural retrieval preserve graph organization when constructing context
\citep{Edge2024GraphRAG,mavromatis2024gnn-rag}.  KGQA, path-reasoning systems,
and graph agents make entities, relations, paths, and graph actions explicit
\citep{yasunaga2021qa-gnn,jiang2023structgpt,Kim2023kg-gpt,Luo2024RoG,
Sun2023ThinkonGraphDA,Jiang2024KG-agent}.  These methods differ substantially in
architecture, but they share a basic exposure of evidence-bearing objects.  The
question for this paper is whether such an object supports a trajectory, not
whether it reveals the hidden computation of the model.

Existing systems also differ in how they use graph evidence.  A retrieved
community can be topically relevant while a generated trajectory uses the wrong
entity, reverses a relation, or bridges two facts that were never observed in the
same evidence state.  A valid KG path may serve as context, a reranking feature,
or a hard admissibility requirement, and those choices carry different coverage
risks.  Constrained decoding enforces admissibility during generation
\citep{hokamp2017constrained}; KG-grounded and faithful generation ask whether
produced content is traceable to evidence \citep{Luo2024RoG,sui2025okgqa}.  What
often remains implicit is the status of absence: it may indicate contradiction,
mere missingness, or a schema mismatch.  The present theory treats those cases as
different evidence states rather than as variants of a single failure label.

The geometric language has several precedents.  KG embeddings provide relation
maps such as TransE translations \citep{bordes2013translating}, and energy-based
learning supplies a general vocabulary for compatibility
\citep{lecun2006ebm}.  Graph signal processing and graph neural networks motivate
Laplacian smoothness \citep{shuman2013gsp,defferrard2016gcn}, but typed directed
KGs require a different invariance: an author transforms into a paper
representation under \textsc{writes} rather than merely resembling the paper.
Posterior regularization and logic-regularized learning already show how soft
constraints can deform probabilistic predictions
\citep{ganchev2010posterior,hu2016logic}.  Our setting differs in the source and
meaning of the constraint: it is induced by an incomplete evidence state, and it
acts on observable trajectories.

Surveys organize KG-enhanced LLM work by task, pipeline, or KG role
\citep{pan2024llmkg,procko2024graph,fan2024rag_survey}.  The perspective here is
orthogonal to those taxonomies.  It asks what evidence object is constructed,
what operation is applied to a candidate trajectory, and how strongly the
evidence-induced energy deforms the LLM prior.  That lens is intended to expose
information assumptions shared across architectures rather than to rank existing
systems.

\section{Observable Trajectories and Evidence States}
\label{sec:setup}

\subsection{Textual and semantic trajectories}

\begin{definition}[Observable reasoning trajectory]
For a query $q\in\Zspace$, an observable reasoning trajectory is a finite
sequence
\[
  \rho=(z_0,z_1,\ldots,z_T)\in\Tspace(q),
\]
where the $z_t$ are utterances, extracted claims, tool observations, or graph
actions and $z_T$ contains a final answer or decision.  The object is observable
or reconstructed; it is not assumed to be the model's hidden causal computation.
\end{definition}

Let $\Cand(q)\subset\Tspace(q)$ be a finite candidate set.  The finiteness
assumption makes the posterior results elementary and transparent; extensions to
countable or continuous spaces require the usual integrability and absolute
continuity conditions.  The LLM prior is
\begin{equation}
 p_\theta(\rho\mid q),\qquad \rho\in\Cand(q),
 \qquad \sum_{\rho\in\Cand(q)}p_\theta(\rho\mid q)=1.
 \label{eq:trajectory-prior}
\end{equation}
It may be obtained by sampling, beam normalization, self-consistency, or an
explicit candidate generator.  A zero prior cannot be repaired by subsequent
KL-regularized grounding; candidate generation therefore remains an upstream
support condition.

Let $(\Mspace,\dM)$ be a semantic metric space and
$\phi:\Zspace\to\Mspace$ an operational representation map.  The semantic
trajectory is $X_\rho=(x_0,\ldots,x_T)$ with $x_t=\phi(z_t)$.  For trajectories
of equal length one may use
\begin{equation}
 d_{\mathrm{prod}}(X_\rho,X_{\rho'})^2
 =\sum_{t=0}^{T}a_t\dM(x_t,x_t')^2,
 \qquad a_t>0.
 \label{eq:product-trajectory-distance}
\end{equation}
The paper does not require $\Mspace$ to be Euclidean or claim that $\phi$
exhausts the semantics of a claim.

For variable-length trajectories, let $\mathcal A(\rho,\rho')$ be a declared set
of admissible alignments and define
\begin{equation}
 d_{\mathcal A}(\rho,\rho')
 =\inf_{A\in\mathcal A(\rho,\rho')}
 \left\{\sum_{(i,j)\in A}c(z_i,z_j')+\operatorname{gap}(A)\right\}.
 \label{eq:alignment-cost}
\end{equation}
This template includes dynamic time warping over claim embeddings, edit costs on
extracted claims, discrete Fr\'echet alignments, optimal transport over claim
sets, and alignment of entity--relation events.  We call
Eq.~(\ref{eq:alignment-cost}) a cost unless the selected construction satisfies
symmetry, identity of indiscernibles, and a triangle inequality.  The posterior
theory requires only a well-defined energy on $\Cand$; projection statements
require a genuine metric or adequate topology.

\subsection{Latent knowledge and observed evidence}

Let $\Gstar$ denote a latent complete knowledge world used only to state
epistemic properties.  It is not available to the grounding algorithm.  The
system observes
\begin{equation}
 R\sim\mathcal O(\Gstar,q),
 \label{eq:evidence-process}
\end{equation}
where the observation process may include KG snapshot selection, retrieval,
entity linking, schema matching, temporal filtering, and provenance assessment.
When graph structure must be named explicitly, write
\[
 \GR=(\Vset_R,\Relset_R,\Edges_R)\subseteq R
\]
for the retrieved snapshot or subgraph contained in the broader state $R$.
Keeping $R$ distinct from $\GR$ matters: two systems can retrieve the same
triples but attach different linking confidences, timestamps, or provenance.

\begin{definition}[Evidence-relative energy]
An evidence-relative energy is a map
\[
 E_R:\Cand(q)\to\mathbb R\cup\{+\infty\}
\]
whose value depends only on $(q,R,\rho)$ and declared representation and
alignment choices.  Lower values indicate stronger compatibility with $R$.
They do not indicate truth in $\Gstar$.
\end{definition}

The raw path energy may be $+\infty$ when no admissible path exists.  Soft
grounding under incomplete evidence will instead use a finite slack-extended
energy.  We reserve $E_R$ for the total energy used by a theorem and state
explicitly whether it is finite.

\subsection{Four evidence statuses}

For declared tolerances, a trajectory or claim can occupy four operational
states:
\begin{enumerate}
 \item \textbf{supported}: a sufficiently reliable alignment has
 $E_R^{\mathrm{sup}}(\rho)\le\epsilon_s$;
 \item \textbf{contradicted}: reliable evidence provides an incompatible typed
 relation, answer, or temporal claim, represented by an explicit counter-support
 condition;
 \item \textbf{unsupported}: neither adequate support nor reliable
 counter-support is present;
 \item \textbf{out of schema}: the relevant claim cannot be expressed by the KG
 schema or the declared alignment language.
\end{enumerate}
These are evidence statuses.  In particular, ``unsupported'' is compatible with
both a true missing fact and a false fluent claim.

\subsection{Hard feasible sets and soft posteriors}

For a tolerance $\epsilon$, define the hard feasible set
\begin{equation}
 \Feas(q;\epsilon)=\{\rho\in\Cand(q):E_R(\rho)\le\epsilon\}.
 \label{eq:hard-feasible-set}
\end{equation}
When $p_\theta(\Feas\mid q)>0$, hard grounding conditions the prior on this set:
\begin{equation}
 p_H^R(\rho\mid q)
 =\frac{p_\theta(\rho\mid q)\mathbf 1\{\rho\in\Feas\}}
 {p_\theta(\Feas\mid q)}.
 \label{eq:hard-posterior}
\end{equation}
For a finite energy and $\lambda\ge0$, soft grounding instead defines
\begin{equation}
 p_\lambda^R(\rho\mid q)
 =\frac{p_\theta(\rho\mid q)e^{-\lambda E_R(\rho)}}
 {Z_R(\lambda)},
 \qquad
 Z_R(\lambda)=\sum_{\rho'\in\Cand(q)}
 p_\theta(\rho'\mid q)e^{-\lambda E_R(\rho')}.
 \label{eq:soft-posterior}
\end{equation}
Sections~\ref{sec:hard}--\ref{sec:stability} characterize what these two choices
can and cannot establish.

\begin{table}[t]
\centering
\small
\caption{Core notation.}
\label{tab:notation}
\begin{tabular}{p{0.24\linewidth}p{0.68\linewidth}}
\toprule
Symbol & Meaning \\
\midrule
$\Gstar$ & latent knowledge world, unavailable to the algorithm \\
$R$, $\GR$ & observed evidence state and its retrieved KG snapshot \\
$\Cand(q)$ & finite candidate set of observable trajectories \\
$\rho$, $X_\rho$ & textual trajectory and semantic realization \\
$p_\theta(\rho\mid q)$ & LLM trajectory prior on $\Cand(q)$ \\
$E_R(\rho)$ & evidence-relative trajectory energy \\
$\SupportRegion(q;\epsilon)$ & semantic support region induced by $R$ \\
$\Feas(q;\epsilon)$ & hard feasible trajectory set \\
$p_\lambda^R(\rho\mid q)$ & soft grounded posterior \\
$\lambda$, $\sigma_R$, $\tau$ & grounding strength, finite slack, path temperature \\
\bottomrule
\end{tabular}
\end{table}

\section{Axioms for Evidence-Relative Grounding}
\label{sec:axioms}

Let a grounding operator map $(q,R,p_\theta)$ to a posterior
$\pi_R\in\Delta(\Cand(q))$ and, when needed, to an action in
\{\textsc{answer}, \textsc{abstain}, \textsc{retrieve}, \textsc{verify}\}.
The axioms below describe the evidential commitments made by open-world
grounding.  They are modeling requirements for the present analysis, not a claim
that every KG-enhanced system already satisfies them.

\paragraph{Axiom 1: Evidence relativity.}
If two latent worlds produce the same $(q,R,p_\theta)$, the grounding operator
returns the same posterior and decision.  Formally, grounding is measurable with
respect to the observed information, not $\Gstar$.  Grounding claims are therefore
indexed by the evidence state relative to which support is assessed.

\paragraph{Axiom 2: Non-falsity of absence.}
If $\rho$ is unsupported but not contradicted and
$p_\theta(\rho\mid q)>0$, missing support alone is not a semantic judgment of
falsehood.  A compatible operator either preserves positive candidacy,
$\pi_R(\rho)>0$, or routes the case to abstention, renewed retrieval, or
verification.

\paragraph{Axiom 3: Direction sensitivity.}
Whenever $r(u,v)$ and $r(v,u)$ have different meanings, the evidence energy
distinguishes them.  In a relational realization this requires, for some relevant
$u,v$,
\[
 \dM(T_r\eta(u),\eta(v))
 \ne \dM(T_r\eta(v),\eta(u)).
\]
If an inverse relation is defined, the reversed transition is evaluated using
$T_{r^{-1}}$, not by symmetrizing $T_r$.

\paragraph{Axiom 4: Contradiction dominance.}
Reliable counter-evidence carries a larger penalty than mere absence.  A
convenient form is the existence of $\kappa_R>0$ such that, after normalization,
\[
 \rho_b\in B_R,\ \rho_u\in U_R
 \quad\Longrightarrow\quad
 E_R(\rho_b)\ge E_R(\rho_u)+\kappa_R,
\]
where $B_R$ and $U_R$ are the contradicted and unsupported states.

\paragraph{Axiom 5: Minimal intervention.}
Grounding changes the LLM prior only through the declared evidence objective.
Posterior-level minimal intervention is represented by
\[
 \pi_R\in\arg\min_{\pi\ll p_\theta}
 \left\{\mathrm{KL}(\pi\|p_\theta)
 +\lambda_R\mathbb E_\pi[E_R]\right\}.
\]
This is a statement about distributional deformation, not about textual edit
distance.

\paragraph{Axiom 6: Coverage awareness.}
Constraint strength depends on evidence coverage or trust.  If $c_R$ is a
declared coverage estimate, one possible calibration is
\[
 \lambda_R=\lambda(c_R),\qquad \sigma_R=\sigma(c_R),
 \qquad \lambda'(c),\sigma'(c)\ge0,
\]
where $\sigma_R$ is the cost of missing support.  Larger $\sigma_R$ treats
absence as more costly, so such a calibration is appropriate only when the
observation process is trusted to have high coverage.  Low coverage keeps absence
weak evidence.

\paragraph{Axiom 7: Stability under evidence perturbation.}
For finite energies, bounded evidence-induced energy changes induce bounded
posterior movement.  Section~\ref{sec:stability} proves the representative bound
\[
 \osc_{\supp p_\theta}(E_R-E_{R'})\le\omega
 \quad\Longrightarrow\quad
 \TV(p_\lambda^R,p_\lambda^{R'})
 \le\tanh\!\left(\frac{\lambda\omega}{2}\right).
\]

\paragraph{Axiom 8: Abstention consistency.}
If the best candidate is unsupported but not contradicted, or if its support
margin is below a declared threshold, the decision space includes abstention,
renewed retrieval, or verification.  Insufficient evidence is not silently
converted into either a positive answer or a hard negative.

\begin{table}[t]
\centering
\small
\caption{Axiomatic compatibility under incomplete evidence.  A triangle denotes
compatibility only with an additional typed, coverage-aware, or selective layer.}
\label{tab:axiom-compatibility}
\begin{tabular}{lccc}
\toprule
Axiom & Hard set & Soft, no slack & Soft, finite slack \\
\midrule
Evidence relativity & $\checkmark$ & $\checkmark$ & $\checkmark$ \\
Non-falsity of absence & $\times$ & $\times$ if $+\infty$ & $\checkmark$ \\
Direction sensitivity & $\triangle$ & $\triangle$ & $\triangle$ \\
Contradiction dominance & $\times$ & $\triangle$ & $\checkmark$ \\
Minimal intervention & $\times$ in general & $\checkmark$ & $\checkmark$ \\
Coverage awareness & $\triangle$ & $\triangle$ & $\checkmark$ \\
Perturbation stability & $\times$ & $\triangle$ & $\checkmark$ \\
Abstention consistency & $\triangle$ & $\triangle$ & $\checkmark$ \\
\bottomrule
\end{tabular}
\end{table}

Taken together, the axioms explain why finite slack is a natural default under
incomplete open-world evidence, while leaving room for hard constraints when
their completeness assumptions are explicit.  Direction sensitivity and
abstention also depend on energy and decision constructions beyond exponential
reweighting alone.

\section{KG-Induced Relational Geometry}
\label{sec:geometry}

The role of geometry is operational.  It provides distances and
relation-conditioned compatibility functions with which a trajectory can be
compared to the evidence in $R$.  It does not assert that an LLM internally
reasons in a particular embedding space.

\subsection{Entity anchors and typed residuals}

Let the observed snapshot be the typed directed graph
$\GR=(\Vset_R,\Relset_R,\Edges_R)$.  A geometric realization consists of an
entity map $\eta:\Vset_R\to\Mspace$ and relation maps
$\{T_r:\Mspace\to\Mspace\}_{r\in\Relset_R}$.  A retrieved triple
$e=(u,r,v)$ has residual
\begin{equation}
 c_{\mathrm{rel}}(e;R)
 =\dM(T_r\eta(u),\eta(v))+\xi(-\log w_e),
 \qquad w_e\in(0,1],\ \xi\ge0,
 \label{eq:typed-residual}
\end{equation}
where $w_e$ represents declared provenance, extraction confidence, or temporal
reliability.  The probabilistic reading
$w_e\approx\Pr(e\text{ reliable}\mid\operatorname{provenance}(e),R)$ is optional
and valid only when such a model is calibrated.  Otherwise $-\log w_e$ is simply
a reliability cost.

Let $\psi_i(\rho)\in\Mspace$ be a claim- or mention-level representation selected
for grounding.  An assignment $a$ maps selected claims to entities or edges in a
path
\[
 \gamma=(v_0,r_1,v_1,\ldots,r_k,v_k),
 \qquad (v_{i-1},r_i,v_i)\in\Edges_R.
\]
Its entity-alignment cost is
\begin{equation}
 c_{\mathrm{ent}}(X_\rho,a)
 =\sum_{(i,v)\in a}\dM(\psi_i(\rho),\eta(v)).
 \label{eq:entity-alignment}
\end{equation}
The path energy is
\begin{equation}
 E_{\mathrm{path}}(X_\rho,\gamma,a;R)
 =c_{\mathrm{ent}}(X_\rho,a)
 +\alpha\sum_{i=1}^{k}c_{\mathrm{rel}}((v_{i-1},r_i,v_i);R),
 \qquad \alpha\ge0.
 \label{eq:path-energy-theory}
\end{equation}
All coefficients calibrate quantities whose raw units need not agree.  They are
evidence-state-dependent trust parameters, not universal constants.

For a retrieved path family $\Pathset_q(R)$ and permissible assignments
$\Alignset(\rho,\gamma)$, the best-path energy is
\begin{equation}
 E_R^{\min}(\rho)
 =\min_{\gamma\in\Pathset_q(R)}
 \min_{a\in\Alignset(\rho,\gamma)}
 E_{\mathrm{path}}(X_\rho,\gamma,a;R),
 \label{eq:min-path-energy}
\end{equation}
with value $+\infty$ when no admissible pair exists.  This raw convention is
useful for defining hard support but must be replaced by finite slack before
applying the finite-energy posterior theory.

\paragraph{Relevance is not support.}
A GraphRAG retriever can return a baseball community containing Lou Seal and the
Giants while a generated step reverses \textsc{mascotOf}.  Community relevance
does not constrain the direction of the transition.  A typed residual can assign
small energy to
$\textsc{Lou Seal}\xrightarrow{\textsc{mascotOf}}\textsc{Giants}$ and large
energy to its reversal even when both utterances are semantically similar.

\subsection{Support regions}

\begin{definition}[Evidence-relative support region]
For tolerance $\epsilon\ge0$, define
\begin{equation}
 \SupportRegion(q;\epsilon)
 =\{X_\rho:\rho\in\Cand(q),\ E_R(\rho)\le\epsilon\}.
 \label{eq:support-region}
\end{equation}
When the energy is computed through the semantic realization $X_\rho$, we write
$E_R(X_\rho;q)$ only as shorthand for $E_R(\rho)$.
\end{definition}

\begin{proposition}[Basic support-region properties]
\label{prop:support-region}
For fixed $(q,R)$:
\begin{enumerate}
 \item if $\epsilon_1\le\epsilon_2$, then
 $\SupportRegion(q;\epsilon_1)\subseteq\SupportRegion(q;\epsilon_2)$;
 \item if $E_R(\cdot;q)$ is lower semicontinuous on the semantic trajectory
 space, then $\SupportRegion(q;\epsilon)$ is closed;
 \item if $\SupportRegion(q;\epsilon)$ is nonempty and compact, every trajectory
 $X$ has at least one metric projection onto it.
\end{enumerate}
The projection need not be unique without additional convexity or geodesic
assumptions.
\end{proposition}

\begin{proof}
The first statement is inclusion of sublevel sets.  The second is the sublevel-set
characterization of lower semicontinuity.  For the third, the continuous function
$Y\mapsto d_{\mathcal A}(X,Y)$ attains a minimum on a nonempty compact set.
\end{proof}

The distance $\dist(X_\rho,\SupportRegion(q;\epsilon))$ is an evidence-support
signal, not a truth signal.  A true trajectory may be distant because its support
edge was not retrieved, and a false trajectory may be close because $R$ contains
a noisy edge.  Different evidence states generally induce different regions.

\subsection{Entropic aggregation over alternative paths}

Best-path selection is itself brittle when several explanations are plausible.
Let
\[
 e_\gamma(X_\rho)
 =\min_{a\in\Alignset(\rho,\gamma)}
 E_{\mathrm{path}}(X_\rho,\gamma,a;R)
\]
and let $\omega\in\Delta(\Pathset_q(R))$ be a path reliability or relevance
prior.  For $\tau>0$, define
\begin{equation}
 E_R^{\mathrm{path}}(\rho)
 =-\tau\log\sum_{\gamma\in\Pathset_q(R)}
 \omega_\gamma\exp\{-e_\gamma(X_\rho)/\tau\}.
 \label{eq:entropic-path-energy}
\end{equation}

\begin{proposition}[Variational path aggregation]
\label{prop:path-variational}
If $\Pathset_q(R)$ is finite and nonempty, $\tau>0$, and $e_\gamma(X)$ is finite
on the support of $\omega$, then
\begin{align}
 E_R^{\mathrm{path}}(\rho)
 =\min_{\nu\in\Delta(\Pathset_q(R)):\,\nu\ll\omega}
 \left\{
 \mathbb E_{\gamma\sim\nu}[e_\gamma(X_\rho)]
 +\tau\mathrm{KL}(\nu\|\omega)
 \right\}. 
 \label{eq:path-variational}
\end{align}
The unique optimizer satisfies
\[
 \nu_\rho^\star(\gamma)
 =\frac{\omega_\gamma e^{-e_\gamma(X_\rho)/\tau}}
 {\sum_{\gamma'}\omega_{\gamma'}e^{-e_{\gamma'}(X_\rho)/\tau}}.
\]
\end{proposition}

\begin{proof}
For any $\nu\ll\omega$, expand
$\tau\mathrm{KL}(\nu\|\nu_\rho^\star)$.  Substitution of the normalizer gives
\[
 \mathbb E_\nu[e_\gamma(X_\rho)]
 +\tau\mathrm{KL}(\nu\|\omega)
 =\tau\mathrm{KL}(\nu\|\nu_\rho^\star)
 -\tau\log\sum_\gamma\omega_\gamma e^{-e_\gamma(X_\rho)/\tau}.
\]
Nonnegativity of KL divergence yields Eq.~(\ref{eq:path-variational}) and
uniqueness.
\end{proof}

The temperature $\tau$ measures path uncertainty.  As $\tau\downarrow0$, the
energy approaches the minimum over paths with positive prior; larger $\tau$
retains alternative explanations.  This avoids treating failure of one path as
falsity when alternatives remain.  If all admissible paths are absent,
Eq.~(\ref{eq:entropic-path-energy}) is still infinite and finite slack is needed.

\subsection{Running example}

For the Lou Seal question, an observed path can be written
\[
 \textsc{Lou Seal}
 \xrightarrow{\textsc{mascotOf}}
 \textsc{San Francisco Giants}
 \xrightarrow{\textsc{mostRecentWorldSeriesWin}}
 \textsc{2014 World Series}.
\]
Entity costs penalize drift to another mascot or team; typed residuals penalize
relation reversal; path aggregation allows alternative reliable provenance; and
the support region contains trajectories whose aligned energy is below a declared
tolerance.  None of these statements turns the retrieved path into an oracle of
truth.  They specify what the current evidence supports.

\section{The Limits of Hard Grounding under Incomplete Evidence}
\label{sec:hard}

Hard grounding conditions the trajectory prior on the evidence-relative set
$\Feas$.  When coverage is complete for the target relation, this can remove
inadmissible candidates without harming a correct one.  Under open-world
incompleteness, however, feasible-set membership cannot identify the truth of an
unsupported claim.

\subsection{An open-world impossibility}


\begin{theorem}[Non-identifiability of unsupported trajectories]
\label{thm:hard-impossibility}
Fix a query $q$, a candidate trajectory $\rho$, and an observation map
$\mathsf O_q$ from latent knowledge worlds to evidence states.  Let
$\mathsf T(W,\rho)\in\{0,1\}$ denote an external truth predicate for $\rho$ in
the latent world $W$, which is not observed by the grounding rule.  Suppose
there exist two worlds $W_1,W_2$ such that
\[
    \mathsf O_q(W_1)=\mathsf O_q(W_2)=R,
\]
while
\[
    \mathsf T(W_1,\rho)=1,\qquad
    \mathsf T(W_2,\rho)=0,
\]
and $\rho$ is unsupported but not explicitly contradicted in the common evidence
state $R$.

Let $h(q,R,\rho)\in\{0,1\}$ be any deterministic hard-grounding rule that is
measurable with respect to the observed information $(q,R,\rho)$, where
$1$ means retain and $0$ means reject.  Then $h$ cannot simultaneously satisfy
the following two requirements:
\[
\mathsf T(W,\rho)=1\ \text{and $\rho$ is unobserved}
    \Longrightarrow h(q,\mathsf O_q(W),\rho)=1,
\]
and
\[
\mathsf T(W,\rho)=0\ \text{and $\rho$ is unsupported}
    \Longrightarrow h(q,\mathsf O_q(W),\rho)=0.
\]
The same conclusion holds for randomized hard-grounding rules if the two
requirements are imposed almost surely.
\end{theorem}

\begin{proof}
The two latent worlds induce the same observable input:
\[
    (q,\mathsf O_q(W_1),\rho)=(q,\mathsf O_q(W_2),\rho)=(q,R,\rho).
\]
Since $h$ uses no information beyond this observed input, it must return the
same decision in both worlds.  If this common decision is $1$, then the rule
retains $\rho$ in $W_2$, where $\rho$ is false and unsupported.  Hence it fails
to reject every false unsupported trajectory.  If the common decision is $0$,
then the rule rejects $\rho$ in $W_1$, where $\rho$ is true but unobserved.
Hence it fails to retain every true-but-unobserved trajectory.  Therefore the
two requirements cannot both hold.

For a randomized rule, let $H(q,R,\rho)$ be the Bernoulli decision generated
from the observed input.  Because the conditional distribution of
$H(q,R,\rho)$ is identical in $W_1$ and $W_2$, it cannot satisfy
$H=1$ almost surely in $W_1$ and $H=0$ almost surely in $W_2$.  Thus the
almost-sure versions of the two requirements are also incompatible.
\end{proof}

For the running example, let $R$ contain
$\textsc{Lou Seal}\xrightarrow{\textsc{mascotOf}}\textsc{Giants}$ but omit the
championship edge.  In $W_1$, the latent world contains the missing 2014 edge; in
$W_2$, it does not or points elsewhere.  The observed evidence is identical:
both worlds expose the mascot edge and hide the answer-bearing edge.  Rejecting
the 2014 trajectory is wrong in $W_1$, whereas accepting it as established fact
is wrong in $W_2$.  The defensible evidence label is therefore
``unsupported,'' followed by verification, renewed retrieval, or abstention.
Finite slack preserves the candidate for such follow-up actions; it does not
certify truth.  Figure~\ref{fig:lou-seal-nonidentifiability} visualizes this
two-world indistinguishability.

\begin{figure}[!htbp]
\centering
\IfFileExists{Figure/Example.png}{%
    \includegraphics[width=0.85\linewidth]{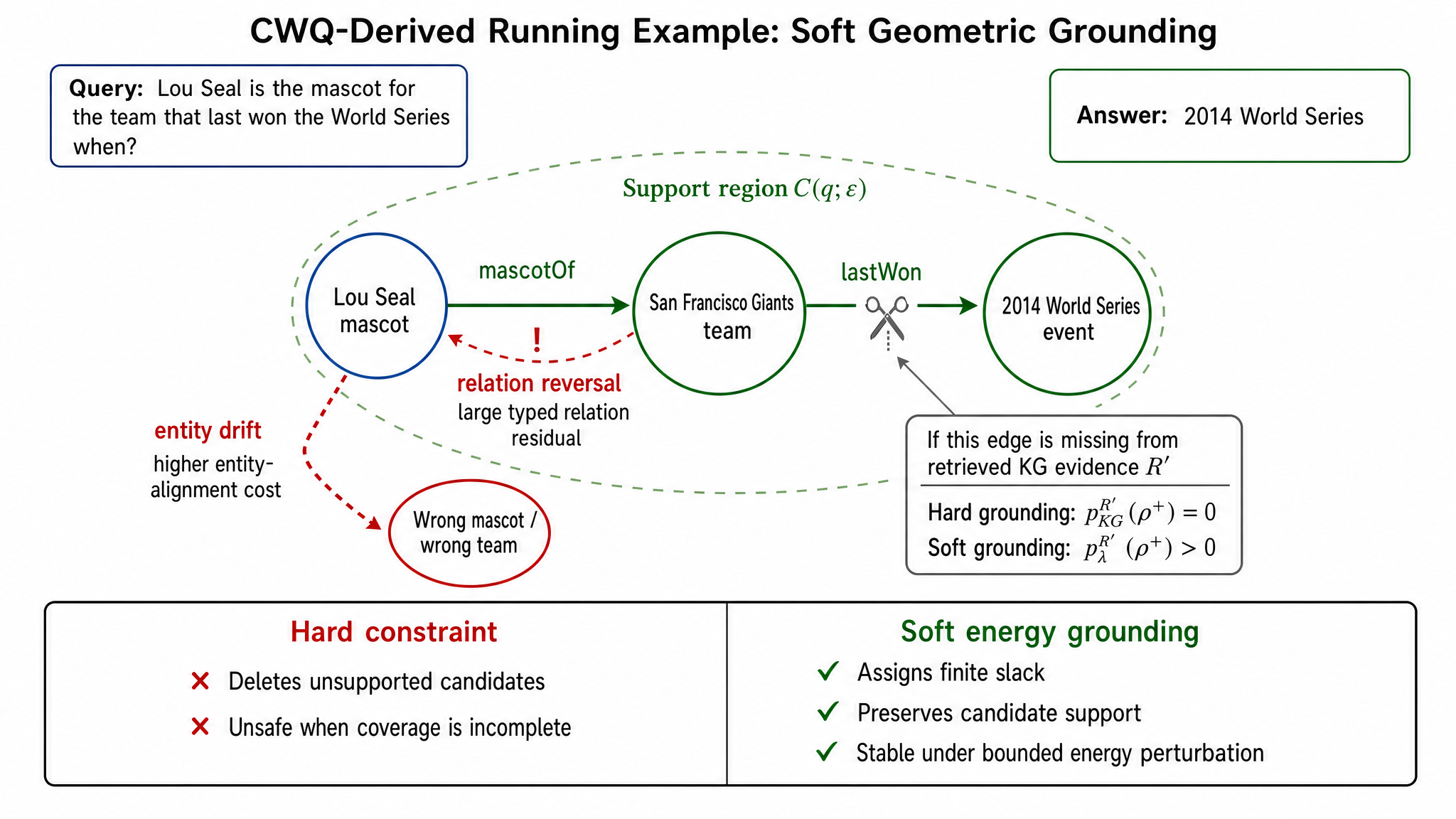}%
}{%
    \fbox{\parbox{0.92\linewidth}{\centering
    Open-world non-identifiability in the Lou Seal example.}}%
}
\caption{Lou Seal non-identifiability: two latent worlds induce the same
observed evidence state while disagreeing on the missing championship edge.
Finite slack preserves the candidate for renewed retrieval, verification, or
abstention without certifying truth.}
\label{fig:lou-seal-nonidentifiability}
\end{figure}
\FloatBarrier

The obstruction is informational rather than algorithmic: no stronger search
procedure or more expressive embedding can recover a distinction erased by the
observation state.

This limitation leaves room for hard constraints when their premises are met:
the relevant relation may be certified closed-world, the snapshot may be
schema-complete for the task, an external oracle may be consulted, or the
constraint may enforce syntax or safety rather than epistemic truth.  These are
applicability conditions, not technical footnotes.

\subsection{Support collapse and coverage}

\begin{proposition}[Hard support collapse]
\label{prop:hard-collapse}
Let $p_\theta(\rho^\dagger\mid q)>0$ for a trajectory
$\rho^\dagger\notin\Feas$.  Whenever the hard posterior in
Eq.~(\ref{eq:hard-posterior}) is defined,
\[
 p_H^R(\rho^\dagger\mid q)=0.
\]
If a finite energy satisfies $E_R(\rho^\dagger)<\infty$, then for every finite
$\lambda\ge0$, the soft posterior in Eq.~(\ref{eq:soft-posterior}) satisfies
$p_\lambda^R(\rho^\dagger\mid q)>0$.
\end{proposition}

\begin{proof}
Hard conditioning multiplies the prior mass by
$\mathbf 1\{\rho^\dagger\in\Feas\}=0$.  Under finite energy, the soft numerator
$p_\theta(\rho^\dagger\mid q)e^{-\lambda E_R(\rho^\dagger)}$ is strictly
positive, and the finite candidate set has a positive finite normalizer.
\end{proof}

Finite soft mass is not endorsement.  It merely prevents missing evidence from
irreversibly deleting a candidate that may require additional information.

Let $\mathcal P^\dagger(q)$ be a declared family of sufficient support paths for a
correct trajectory $\rho^\dagger$.  Define
\begin{equation}
 \operatorname{Cov}(q)
 =\Pr_R[\exists\gamma\in\mathcal P^\dagger(q):\gamma\subseteq R].
 \label{eq:path-coverage}
\end{equation}

\begin{proposition}[Coverage bottleneck]
\label{prop:coverage-bottleneck}
Suppose a hard path-complete rule can retain $\rho^\dagger$ only if $R$ contains
at least one path in $\mathcal P^\dagger(q)$.  Then
\[
 \Pr_R[\rho^\dagger\in\Feas]
 \le \operatorname{Cov}(q).
\]
For one required $k$-edge path whose edges are retrieved independently with
probabilities $c_1,\ldots,c_k$, its coverage is $\prod_{i=1}^{k}c_i$.
\end{proposition}

\begin{proof}
The retention event is a subset of the event that at least one sufficient path
is observed.  The product formula follows only under the stated independence
assumption.
\end{proof}

The phrase ``at least one'' is important: bounding retention by coverage of one
named path is invalid when alternative sufficient paths exist.  Real failures
are also correlated.  Missing a seed entity can remove a neighborhood;
missing-not-at-random evidence can suppress rare relation types; temporal
snapshots can be stale; and schema incompleteness can make a claim inexpressible.
Coverage is therefore a theoretical bottleneck of hard epistemic grounding, not
merely a retrieval implementation detail.

\paragraph{Counterexample: a correct answer with a missing edge.}
If the observed Lou Seal subgraph omits only the Giants--2014 edge, a hard rule
removes the correct two-hop trajectory.  A finite-slack posterior assigns it a
penalty rather than a zero probability and can route it to renewed retrieval.
This is precisely the support-collapse mechanism; it does not license answering
from every unsupported trajectory.

\section{Soft Grounding as Variational Posterior Deformation}
\label{sec:variational}

Hard support regions are useful for stating what the observed graph directly
licenses, but they are too brittle as a full account of KG-enhanced reasoning.
The language model already induces a prior distribution over possible
trajectories, and incomplete graph evidence normally changes this distribution
rather than replacing it by a zero-one filter.  We therefore treat grounding as a
posterior deformation problem: trajectories with lower evidence energy become
more likely, while finite-energy alternatives remain in the support unless the
method explicitly takes a hard-limit regime.

Let $P_0$ denote the baseline probability measure induced by the LLM prior
$p_\theta$; equivalently, $P_0(\rho\mid q)=p_\theta(\rho\mid q)$ on $\Cand(q)$. Let
$E_R(\rho)\in[0,\infty)$ be an evidence energy.  For $\lambda\ge0$, consider
\begin{equation}
 P_\lambda^R
 \in
 \arg\min_{P\ll P_0}
 \left\{
 \mathrm{KL}(P\|P_0)
 +\lambda\,\mathbb E_{\rho\sim P}[E_R(\rho)]
 \right\}.
 \label{eq:kl-grounding-objective}
\end{equation}
The absolute-continuity constraint records a modest but important choice: the KG
evidence reweights trajectories available to the model prior; it does not invent
new trajectories outside the modeled candidate space.

\begin{theorem}[KL-regularized characterization]
\label{thm:variational-soft}
Assume $\Cand(q)$ is finite or countable, $P_0(\rho\mid q)>0$ on its support,
$E_R$ is finite on that support, and
\[
 Z_R(\lambda)
 =
 \sum_{\rho}
 P_0(\rho\mid q)\exp\{-\lambda E_R(\rho)\}<\infty .
\]
For $\lambda\ge0$, the optimization problem in
Eq.~(\ref{eq:kl-grounding-objective}) has the unique solution
\begin{equation}
 P_\lambda^R(\rho\mid q)
 =
 \frac{
 P_0(\rho\mid q)\exp\{-\lambda E_R(\rho)\}
 }{
 Z_R(\lambda)
 } .
 \label{eq:gibbs-grounded-posterior}
\end{equation}
Moreover, for every $P\ll P_0$,
\begin{equation}
 \mathrm{KL}(P\|P_0)
 +\lambda\,\mathbb E_P[E_R]
 =
 \mathrm{KL}(P\|P_\lambda^R)
 -\log Z_R(\lambda).
 \label{eq:kl-decomposition}
\end{equation}
\end{theorem}

\begin{proof}
For every $\rho$ in the support of $P_0$,
\[
 \log\frac{P(\rho)}{P_\lambda^R(\rho)}
 =
 \log\frac{P(\rho)}{P_0(\rho)}
 +\lambda E_R(\rho)
 +\log Z_R(\lambda).
\]
Taking expectation under $P$ and rearranging yields
Eq.~(\ref{eq:kl-decomposition}).  Since KL divergence is nonnegative and
vanishes only when $P=P_\lambda^R$, the minimizer is unique.
\end{proof}

This gives soft grounding the same minimum-information interpretation as
posterior regularization: among all distributions that pay for evidence
incompatibility, the selected posterior moves as little as possible from the LLM
prior in KL geometry.  The parameter $\lambda$ is a grounding-strength parameter.
When $\lambda=0$, the posterior reduces to the LLM prior; larger $\lambda$ gives
greater influence to the evidence energy.

Finite slack is the technical bridge between graph absence and open-world
uncertainty.  When no retrieved path supports a trajectory, assigning infinite
energy forces immediate rejection.  A slackened energy instead records the cost
of unsupportedness without turning absence into contradiction:
\begin{equation}
 E_R^{\mathrm{slack}}(\rho)
 =\min\{E_R^{\min}(\rho),\kappa_R(\rho)\},
 \qquad 0<\kappa_R(\rho)<\infty .
 \label{eq:finite-slack-energy}
\end{equation}
Here $\kappa_R$ may depend on retrieval coverage, entity-linking confidence,
schema fit, temporal freshness, or application-level tolerance.  It is not a
claim that unsupported trajectories are correct; it only prevents the posterior
from confusing lack of observed support with logical falsity.

\begin{proposition}[Minimum-energy concentration]
\label{prop:min-energy-limit}
Let $\Cand(q)$ be finite and suppose $P_0(\rho\mid q)>0$ for all candidates.
For fixed finite energy $E_R$, the posterior $P_\lambda^R$ concentrates, as
$\lambda\to\infty$, on the set
\[
 M_R=\arg\min_{\rho\in\Cand(q)}E_R(\rho).
\]
Moreover, for every $\rho\in M_R$,
\[
 P_\lambda^R(\rho\mid q)
 \to
 \frac{P_0(\rho\mid q)}{P_0(M_R\mid q)} .
\]
\end{proposition}

\begin{proof}
Let $E_{\min}=\min_\rho E_R(\rho)$ and write
$\Delta_\rho=E_R(\rho)-E_{\min}\ge0$.  Dividing numerator and denominator in
Eq.~(\ref{eq:gibbs-grounded-posterior}) by $e^{-\lambda E_{\min}}$ gives weights
proportional to $P_0(\rho\mid q)e^{-\lambda\Delta_\rho}$.  Non-minimizers have
$\Delta_\rho>0$ and hence relative weights $e^{-\lambda\Delta_\rho}\to0$, while
minimizers retain their prior weights.
\end{proof}



\begin{proposition}[Hard conditioning as an infinite-penalty limit]
\label{prop:hard-limit}
Let
\[
S_R=\{\rho\in\Gamma(q):E_R(\rho)\le\epsilon\},
\]
and define the truncated energy
\[
E_M(\rho)=
\begin{cases}
0, & \rho\in S_R,\\
M, & \rho\notin S_R.
\end{cases}
\]
If $P_0(S_R\mid q)>0$, then the soft posterior induced by $E_M$ satisfies
\[
P_{\lambda,M}^R(\rho\mid q)
\longrightarrow
\frac{
P_0(\rho\mid q)\mathbf 1\{\rho\in S_R\}
}{
P_0(S_R\mid q)
}
\]
whenever $\lambda M\rightarrow\infty$.
\end{proposition}

\begin{proof}
Under the truncated energy $E_M$, every trajectory in $S_R$ retains its prior
weight, whereas every trajectory outside $S_R$ is multiplied by the factor
$e^{-\lambda M}$. Since $P_0(S_R\mid q)>0$, the normalizing constant converges
to $P_0(S_R\mid q)$ as $\lambda M\to\infty$. Therefore,
\[
P_{\lambda,M}^R(\rho\mid q)
\longrightarrow
\frac{
P_0(\rho\mid q)\mathbf1\{\rho\in S_R\}
}{
P_0(S_R\mid q)
},
\]
which is precisely the prior conditioned on the feasible set $S_R$.
\end{proof}

Hard grounding is therefore a boundary case obtained when unsupported
trajectories receive unbounded effective penalty.  In open-world graphs this
boundary is defensible only when completeness assumptions are explicit.

\begin{proposition}[Coverage-and-margin condition for soft MAP selection]
\label{prop:coverage-margin-map}
Let $\rho^\dagger\in\Cand(q)$ be a correct answer-bearing trajectory with
$P_0(\rho^\dagger\mid q)>0$, where correctness is defined by an external task
criterion rather than by the evidence energy $E_R$. Let $R$ be random.
Suppose there exist an evidence event
\[
M\subseteq\mathcal R
\]
and a constant $\Delta>0$ such that, on the event $M$,
\[
E_R(\rho)-E_R(\rho^\dagger)\ge\Delta
\]
for every competing trajectory
$\rho\neq\rho^\dagger$
with
$P_0(\rho\mid q)>0$.
If
\[
\lambda\Delta
>
\max_{\rho\neq\rho^\dagger:\,
P_0(\rho\mid q)>0}
\log
\frac{P_0(\rho\mid q)}
     {P_0(\rho^\dagger\mid q)},
\]
then $\rho^\dagger$ is the unique MAP trajectory of
$P_\lambda^R(\cdot\mid q)$.
Consequently,
\[
\Pr_R
\!\left[
\rho^\dagger
=
\arg\max_{\rho\in\Cand(q)}
P_\lambda^R(\rho\mid q)
\right]
\ge
\Pr_R(M).
\]

Furthermore, suppose a declared sufficient support path for
$\rho^\dagger$
contains $k$ required edges, each of which is retrieved independently with
probability at least $1-\epsilon$. If successful retrieval of all these edges
implies the event $M$, then
\[
\Pr_R(M)\ge(1-\epsilon)^k.
\]
\end{proposition}

\begin{proof}
If no competing trajectory has positive prior mass, then
$\rho^\dagger$ is trivially the unique MAP trajectory.

Otherwise, on the event $M$, every competing trajectory
$\rho$
with positive prior mass satisfies
\[
\log
\frac{P_\lambda^R(\rho\mid q)}
     {P_\lambda^R(\rho^\dagger\mid q)}
=
\log
\frac{P_0(\rho\mid q)}
     {P_0(\rho^\dagger\mid q)}
-
\lambda
\bigl(
E_R(\rho)-E_R(\rho^\dagger)
\bigr)
<0,
\]
where the strict inequality follows from the assumed margin condition.
Hence every positive-prior competitor has strictly smaller posterior mass than
$\rho^\dagger$. Since trajectories with zero prior mass also have zero posterior
mass under KL-regularized grounding, $\rho^\dagger$ is the unique MAP
trajectory.

Therefore,
\[
M
\subseteq
\left\{
\rho^\dagger
=
\arg\max_{\rho\in\Cand(q)}
P_\lambda^R(\rho\mid q)
\right\},
\]
which immediately yields
\[
\Pr_R
\!\left[
\rho^\dagger
=
\arg\max_{\rho\in\Cand(q)}
P_\lambda^R(\rho\mid q)
\right]
\ge
\Pr_R(M).
\]

Under the stated independent-edge retrieval assumption,
all required edges are simultaneously retrieved with probability at least
$(1-\epsilon)^k$. Since this event implies $M$, the final inequality follows.
\end{proof}

The proposition is intentionally conditional.  It does not claim that soft
grounding automatically selects a correct answer.  It says that selection occurs
on evidence states where coverage produces an energy margin large enough to
overcome the LLM prior odds.  The independent-edge condition is only a diagnostic
baseline; correlated retrieval failures should be represented by the
corresponding joint path-coverage probability.

The posterior odds decompose into a prior-odds term and an evidence-margin term.
For any two candidates $\rho_1,\rho_2$ with positive prior probability,
\begin{equation}
 \log\frac{P_\lambda^R(\rho_1\mid q)}{P_\lambda^R(\rho_2\mid q)}
 =
 \log\frac{P_0(\rho_1\mid q)}{P_0(\rho_2\mid q)}
 -\lambda\{E_R(\rho_1)-E_R(\rho_2)\}.
 \label{eq:posterior-odds}
\end{equation}

\begin{corollary}[Pairwise ranking special case]
\label{cor:pairwise-ranking}
For two candidates $\rho_T$ and $\rho_F$ with positive priors,
$P_\lambda^R$ ranks $\rho_T$ above $\rho_F$ if and only if
\[
 \lambda\{E_R(\rho_F)-E_R(\rho_T)\}
 >
 \log
 \frac{P_0(\rho_F\mid q)}
 {P_0(\rho_T\mid q)}.
\]
Equivalently,
\[
 \frac{P_0(\rho_T\mid q)}{P_0(\rho_F\mid q)}
 >
 \exp\{-\lambda(E_R(\rho_F)-E_R(\rho_T))\}.
\]
\end{corollary}

\begin{proof}
Apply Eq.~(\ref{eq:posterior-odds}) with $\rho_1=\rho_T$ and $\rho_2=\rho_F$.
The posterior ranks $\rho_T$ above $\rho_F$ exactly when the resulting log-odds
is positive.
\end{proof}

Whether $\rho_T$ or $\rho^\dagger$ is actually true is external to the energy
model.  For the Lou Seal example, a correctly oriented path lowers the energy of
the trajectory that names the Giants and their 2014 World Series win, while a
relation-reversed trajectory pays a larger typed residual.  If the model prior
strongly favors a generic but unsupported answer, Eq.~(\ref{eq:posterior-odds})
shows how much evidence margin is needed to overturn it.

\begin{table}[t]
\centering
\small
\begin{tabular}{lll}
\toprule
Parameter & Mathematical role & Practical interpretation\\
\midrule
$\lambda$ & grounding-strength parameter & strength of evidence-induced reweighting relative to the LLM prior\\
$\alpha$ & relation-residual scale & trust in typed KG transitions\\
$w_e$ & edge reliability & provenance, extraction, or temporal confidence\\
$\kappa_R$ & unsupportedness slack & cost of missing graph support\\
$\tau$ & path-aggregation temperature & uncertainty over alternative paths\\
\bottomrule
\end{tabular}
\caption{Calibration parameters in the soft-grounding formulation.  They are
evidence-state and system dependent; none is a universal semantic constant.}
\label{tab:calibration}
\end{table}

The calibration view also discourages overinterpretation.  Lower energy means
greater compatibility with the current evidence representation, not greater
truth in an absolute sense.  If edge reliabilities are miscalibrated, if entity
linking is wrong, or if retrieval misses the relevant neighborhood, the posterior
will faithfully reflect those defects.  Soft grounding makes the tradeoff
visible rather than hiding it behind a binary constraint.

The parameters $\beta$ and $\mu$ are used in the projection-style extension in
Appendix~\ref{app:projection} rather than in the main posterior-deformation
results.

\section{Stability, Coverage, and Abstention}
\label{sec:stability}

If graph grounding is evidence relative, then stability becomes part of the
theory rather than an implementation detail.  We ask that small changes in
alignment, relation residuals, reliability weights, or slack costs induce
bounded changes in the grounded posterior.  This requirement does not make the
evidence correct, but it prevents a system from treating small retrieval or
linking perturbations as decisive semantic events.

\begin{proposition}[Component-wise evidence perturbations induce bounded energy perturbations]
\label{prop:component-energy-perturbation}
For every $\rho\in\supp P_0$, let a common nonempty finite index set $I_\rho$
enumerate the path-assignment or slack alternatives compared under evidence
states $R$ and $R'$.  Suppose
\[
 E_R(\rho)=\min_{i\in I_\rho}\Phi_R(\rho,i),
 \qquad
 E_{R'}(\rho)=\min_{i\in I_\rho}\Phi_{R'}(\rho,i),
\]
where all indexed alternative costs are finite, and suppose
\[
 \sup_{\rho\in\supp P_0}
 \sup_{i\in I_\rho}
 |\Phi_R(\rho,i)-\Phi_{R'}(\rho,i)|
 \le \delta .
\]
Then
\[
 \sup_{\rho\in\supp P_0}
 |E_R(\rho)-E_{R'}(\rho)|
 \le \delta .
\]
In particular, suppose every indexed alternative contains at most $m_{\max}$
entity-alignment terms, at most $k_{\max}$ typed relation edges, and at most one
slack term.  If the component changes satisfy
\[
 \left|
 \dM(\psi_t^R(\rho),\eta^R(v))
 -
 \dM(\psi_t^{R'}(\rho),\eta^{R'}(v))
 \right|
 \le
 \epsilon_{\mathrm{ent}},
\]
\[
 \left|
 \dM(T_r^R(\eta^R(u)),\eta^R(v))
 -
 \dM(T_r^{R'}(\eta^{R'}(u)),\eta^{R'}(v))
 \right|
 \le
 \epsilon_{\mathrm{rel}},
\]
\[
 \left|
 -\log w_{urv}^{R}
 +
 \log w_{urv}^{R'}
 \right|
 \le
 \epsilon_w,
\]
and
\[
 \sup_{\rho\in\supp P_0}
 |\kappa_R(\rho)-\kappa_{R'}(\rho)|
 \le
 \epsilon_{\mathrm{slack}},
\]
then the preceding assumption holds with
\[
 \delta
 =
 m_{\max}\epsilon_{\mathrm{ent}}
 +
 \alpha k_{\max}(\epsilon_{\mathrm{rel}}+\xi\epsilon_w)
 +
 \epsilon_{\mathrm{slack}} .
\]
\end{proposition}

\begin{proof}
Fix $\rho$ and choose
$i_R\in\arg\min_{i\in I_\rho}\Phi_R(\rho,i)$.  Then
\[
 E_{R'}(\rho)
 \le
 \Phi_{R'}(\rho,i_R)
 \le
 \Phi_R(\rho,i_R)+\delta
 =
 E_R(\rho)+\delta .
\]
Interchanging $R$ and $R'$ gives
$E_R(\rho)\le E_{R'}(\rho)+\delta$, proving the uniform energy bound.  The
displayed component-wise bound follows by applying the triangle inequality to
each indexed path or slack alternative and using the maximum numbers of
entity-alignment and relation-residual terms.  The slack component changes by
at most $\epsilon_{\mathrm{slack}}$ uniformly over $\rho\in\supp P_0$.
\end{proof}

If two evidence states induce different path-assignment sets, one may compare
them on the union of the two sets and assign finite slack to alternatives absent
from one state.  Finite slack is essential: the proposition does not control
jumps from finite cost to $+\infty$.  The result is not tied to Euclidean
embeddings; it applies to any metric-space realization in which the displayed
residuals are measured by $\dM$.

By Theorem~\ref{thm:soft-stability} below, this component-wise bound implies
\[
 \TV(P_\lambda^R,P_\lambda^{R'})
 \le
 \tanh(\lambda\delta).
\]

The preceding proposition controls how local perturbations in entity alignment, typed relation residuals, reliability weights, and slack costs affect the trajectory energy. The next result translates this energy-level stability into posterior-level stability. Since adding the same constant to all energies does not change the normalized posterior, the relevant quantity is the oscillation of
\(E_R-E_{R'}\), rather than its absolute magnitude.

\begin{theorem}[Posterior stability under evidence perturbation]
\label{thm:soft-stability}
Let $\Cand(q)$ be finite, let $P_0$ be the same prior under two evidence states
$R$ and $R'$, and let $E_R,E_{R'}:\Cand(q)\to\mathbb R$ be finite energies.
Define $P_\lambda^R$ and $P_\lambda^{R'}$ by
\[
 P_\lambda^R(\rho\mid q)
 =
 \frac{P_0(\rho\mid q)\exp\{-\lambda E_R(\rho)\}}{Z_R(\lambda)}
\]
and analogously for $R'$.  On $S=\supp(P_0)$, let
\[
 \omega
 =
 \operatorname{osc}_{S}(E_R-E_{R'})
 =
 \sup_{\rho\in S}\{E_R(\rho)-E_{R'}(\rho)\}
 -
 \inf_{\rho\in S}\{E_R(\rho)-E_{R'}(\rho)\}.
\]
Then, for every $\rho\in S$,
\[
 \left|
 \log
 \frac{P_\lambda^R(\rho\mid q)}
 {P_\lambda^{R'}(\rho\mid q)}
 \right|
 \le
 \lambda\omega,
\]
and
\begin{equation}
 \TV(P_\lambda^R,P_\lambda^{R'})
 \le
 \tanh\left(\frac{\lambda\omega}{2}\right).
 \label{eq:tv-stability}
\end{equation}
In particular, if
\[
 \sup_{\rho\in S}|E_R(\rho)-E_{R'}(\rho)|\le \delta,
\]
then
\[
 \TV(P_\lambda^R,P_\lambda^{R'})
 \le
 \tanh(\lambda\delta).
\]
\end{theorem}

\begin{proof}
Write
\[
 \Delta E(\rho)=E_R(\rho)-E_{R'}(\rho),
\]
and let
\[
 m=\inf_{\rho\in S}\Delta E(\rho),
 \qquad
 M=\sup_{\rho\in S}\Delta E(\rho).
\]
Then $\omega=M-m$.  Direct substitution gives
\[
 \frac{P_\lambda^R(\rho\mid q)}
 {P_\lambda^{R'}(\rho\mid q)}
 =
 \frac{\exp\{-\lambda \Delta E(\rho)\}}
 {\mathbb E_{P_\lambda^{R'}}[\exp\{-\lambda \Delta E\}]}.
\]
Both the numerator and the denominator lie between $e^{-\lambda M}$ and
$e^{-\lambda m}$, so the likelihood ratio lies in
\[
 [e^{-\lambda\omega},e^{\lambda\omega}].
\]
This gives the log-ratio bound.

Let
\[
 L(\rho)=
 \frac{P_\lambda^R(\rho\mid q)}
 {P_\lambda^{R'}(\rho\mid q)} .
\]
Under $P_\lambda^{R'}$, $L\in[a,b]$ with
\[
 a=e^{-\lambda\omega},
 \qquad
 b=e^{\lambda\omega},
\]
and $\mathbb E_{P_\lambda^{R'}}[L]=1$.  Therefore
\[
 \TV(P_\lambda^R,P_\lambda^{R'})
 =
 \mathbb E_{P_\lambda^{R'}}[(L-1)_+]
 \le
 \frac{(b-1)(1-a)}{b-a}
 =
 \tanh\left(\frac{\lambda\omega}{2}\right).
\]
If $\sup_{\rho\in S}|E_R(\rho)-E_{R'}(\rho)|\le\delta$, then
$\omega\le2\delta$, giving
\[
 \TV(P_\lambda^R,P_\lambda^{R'})
 \le
 \tanh(\lambda\delta).
\]
\end{proof}

Together with Proposition~\ref{prop:component-energy-perturbation}, the theorem
shows how bounded perturbations in entity alignment, typed relation residuals,
reliability weights, and finite slack costs induce bounded posterior movement.
The result also exposes a calibration tension: increasing $\lambda$ strengthens
the influence of graph evidence but amplifies sensitivity to energy error.

\subsection{Random evidence processes}

The observed evidence state can be viewed as the outcome of a random process:
retrieval selects documents or graph neighborhoods, entity linking maps mentions
to nodes, extraction or KG lookup proposes typed edges, and temporal or
provenance filters decide which edges remain active.  Let $R\sim\mathcal Q(\cdot
\mid q,\Gstar)$ denote this process given the query and the latent complete
world graph.  A trajectory may be unsupported under $R$ because it is false, but
also because the process failed to expose the relevant evidence.

This factorization clarifies why coverage cannot be summarized by an independent
edge-retention rate.  Missing a seed entity may remove an entire neighborhood;
missingness can depend on popularity, relation type, language, or provenance;
temporal snapshots can be stale; and schema mismatch can make a true claim
inexpressible even under exhaustive traversal.  These are not edge-level noise
terms added after the fact.  They shape the candidate evidence state before the
grounding energy is ever evaluated.

Consequently, unsupportedness is best read as a statement about the pair
$(R,\rho)$, not directly about the world.  High confidence from absence is
therefore justified only when the evidence process is declared complete for the
relevant slice of the graph: the entities are covered, the relation types are
closed, the temporal scope matches the question, and the extraction or lookup
procedure has known recall.  Without such a coverage claim, finite slack is not
a modeling convenience but the mathematically honest representation of
open-world uncertainty.

\subsection{Selective grounding}

Evidence-relative systems also need a way to abstain from strong grounded
claims.  Let $s_R(\rho)$ be a support score derived from the posterior, an
energy margin, or a calibrated confidence model.  A selective rule emits a
grounded answer only when
\begin{equation}
 \max_{\rho\in\Cand(q)} s_R(\rho)\ge \theta,
 \qquad
 \widehat{\rho}
 \in\arg\max_{\rho\in\Cand(q)}s_R(\rho),
 \label{eq:selective-rule}
\end{equation}
and otherwise returns an abstention or an answer explicitly marked as not
verified by the graph evidence.

\begin{proposition}[Selective risk under calibrated support]
\label{prop:selective-risk}
Suppose $s_R(\rho)$ is calibrated in the sense that
\[
 \Pr(Y=1\mid s_R(\widehat{\rho})=s,\ \text{emit})=s,
\]
where $Y=1$ denotes the external event that the emitted trajectory is correct.
Then the conditional error rate among emitted answers with
$s_R(\widehat{\rho})\ge\theta$ is at most $1-\theta$.
\end{proposition}

\begin{proof}
By calibration,
\[
\Pr(Y=0\mid s_R(\widehat{\rho})\ge\theta,\text{emit})
=\mathbb E[1-s_R(\widehat{\rho})\mid s_R(\widehat{\rho})\ge\theta,\text{emit}]
\le 1-\theta.
\]
\end{proof}

The proposition does not say that graph compatibility automatically yields a
calibrated truth probability.  Calibration is an empirical or domain-modeling
assumption.  The value of the formal statement is narrower: if a system chooses
to communicate a scalar support score as confidence, then selective abstention
has a clear risk meaning only after that score has been calibrated against the
external correctness event of interest.

\paragraph{Faithful-generation implication.}
For faithful generation, the relevant output is often not only an answer but a
claim of support: ``the graph verifies this step'', ``the retrieved evidence
entails this relation'', or ``the answer is grounded in the KG.''  Under the
present perspective, such commitments are stronger than ordinary generation.
An unsupported trajectory may still be true, but it is not graph-verified at the
communication layer.  Faithful generation therefore requires a separation among
answer content, graph support, and residual uncertainty in the user-facing
response.

\section{Relation-Aware Regularization}
\label{sec:regularization}

Trajectory-level energy acts at inference time.  Graph structure can also couple
representations during learning or adaptation.  The relevant regularizer needs
to respect relation type and direction rather than assume homogeneous adjacency.

\subsection{Why ordinary Laplacian smoothing is insufficient}

Let $H\in\mathbb R^{n\times d}$ collect node representations.  For a homogeneous
undirected weighted graph with adjacency $A$, degree matrix $D$, and $L=D-A$,
the classical energy is
\begin{equation}
 \operatorname{Tr}(H^\top L H)
 =\frac12\sum_{u,v}A_{uv}\|h_v-h_u\|_2^2.
 \label{eq:classical-laplacian}
\end{equation}
It encodes homophily: low energy makes adjacent representations close.  A KG
edge often connects different semantic roles.  An author need not resemble a
paper, and a country need not resemble its capital.

For a typed directed snapshot, define
\begin{equation}
 \Omega_R(H)
 =\frac12\sum_{(u,r,v)\in\Edges_R}
 w_{urv}\|h_v-T_r(h_u)\|_2^2.
 \label{eq:typed-regularizer}
\end{equation}
The residual encourages agreement with a relation-conditioned prediction, not
identity of adjacent nodes.

\begin{proposition}[Classical Laplacian as a special case]
\label{prop:laplacian-reduction}
Suppose an undirected graph is represented in $\Edges_R$ by both ordered arcs
$(u,v)$ and $(v,u)$, with symmetric weights $w_{uv}=w_{vu}=A_{uv}$, and the sole
relation map is $T_r=I$.  Then
\[
 \Omega_R(H)=\operatorname{Tr}(H^\top L H).
\]
\end{proposition}

\begin{proof}
Under the assumptions,
\[
 \Omega_R(H)=\frac12\sum_{u,v}A_{uv}\|h_v-h_u\|_2^2,
\]
which equals Eq.~(\ref{eq:classical-laplacian}) after expanding the norm and
using $L=D-A$.
\end{proof}

The ordered-arc convention is essential.  If each undirected edge is stored only
once, the prefactor changes or the equality holds only up to a constant.  Thus
ordinary Laplacian smoothing is a homogeneous identity-relation reduction, not
the default form for a KG.

\subsection{Oversmoothing counterexample}

Consider the one-edge KG
\[
 \textsc{Author }a\xrightarrow{\textsc{writes}}\textsc{Paper }p.
\]
The untyped energy $\|h_p-h_a\|^2$ is minimized at $h_p=h_a$ when unconstrained,
erasing the author--paper role distinction.  A translation map
$T_{\textsc{writes}}(x)=x+r_{\textsc{writes}}$ instead gives
\[
 \|h_p-h_a-r_{\textsc{writes}}\|^2,
\]
which is zero at $h_p=h_a+r_{\textsc{writes}}$ and permits systematic role
separation.  Making the untyped penalty soft merely weakens the wrong invariance;
it does not make the invariance relationally appropriate.

\subsection{Direction and reversal}

Suppose a valid forward triple satisfies $h_v=T_r(h_u)$.  Its forward residual is
zero, whereas reusing the same relation in reverse yields
\begin{equation}
 \|h_u-T_r(h_v)\|
 =\|h_u-T_r(T_r(h_u))\|,
 \label{eq:reverse-residual}
\end{equation}
which is generically nonzero.  For $T_r(x)=x+r$, the reversed residual is
$2\|r\|$.  When an inverse relation is available, reversal is evaluated with
$T_{r^{-1}}$, ideally satisfying
$T_{r^{-1}}\circ T_r\approx I$.  We do not claim asymmetry for every map:
identity and involutive maps are explicit exceptions.

For example, the observed edge
$\textsc{France}\xrightarrow{\textsc{capital}}\textsc{Paris}$ does not license
the sentence ``Paris has capital France.''  Symmetric semantic similarity can
miss this error; the typed directed residual can expose it.

\subsection{Connection to KG embedding objectives}

TransE is the special case $T_r(h_u)=h_u+r$, for which the residual is
$\|h_u+r-h_v\|$ \citep{bordes2013translating}.  Linear and neural relation maps
are immediate generalizations.  Bilinear models use compatibility scores such as
$s_r(u,v)=h_u^\top W_rh_v$; an energy can be defined as $-s_r$ or through a
calibrated margin loss rather than forcing the score into Euclidean-distance
notation.  This paper does not propose a new KG embedding objective.  It promotes
typed compatibility from triple representation learning to an evidence-relative
diagnostic and regularizer for observable trajectories.

\paragraph{Scope of representation coupling.}
A small $\Omega_R(H)$ indicates agreement with the declared relation maps on the
observed snapshot.  It does not establish that a generated trajectory used those
relations, that $R$ is correct, or that the final answer is factual.  Auditable
trajectory grounding still requires an explicit link between generated claims
and the representations being regularized.

\section{A Taxonomy of Constraint Regimes and Design Implications}
\label{sec:taxonomy}

The preceding sections separate three questions that are often collapsed in
KG-enhanced LLM systems: what evidence was observed, how strongly that evidence
constrains the model distribution, and what commitment the system makes when it
speaks.  This separation yields a taxonomy of grounding regimes rather than a
single recipe for ``using a KG.''  The same graph can support retrieval,
posterior reweighting, constrained decoding, verification, or abstention,
depending on the coverage assumptions and the communicative goal.

\begin{table}[t]
\centering
\small
\begin{tabular}{p{0.21\linewidth}p{0.27\linewidth}p{0.38\linewidth}}
\toprule
Regime & Formal analogue & Appropriate use\\
\midrule
Relevance retrieval &
Evidence state $R$ expands the model context but imposes no feasibility
constraint &
Graph neighborhoods, documents, or communities are used to expose useful
information; absence of a path carries little negative meaning.\\
\midrule
Soft grounding &
Finite energy and KL-regularized posterior deformation &
Observed evidence changes answer probabilities while preserving prior support
for plausible but unobserved trajectories.\\
\midrule
Hard grounding &
Conditioning on a support region or an infinite-penalty limit &
Defensible when the evidence state is declared complete for the relevant
entities, relations, and time interval.\\
\midrule
Selective grounding &
Emission only above a calibrated support threshold &
Useful when user-facing claims require explicit graph verification or when
unsupported answers can be marked rather than silently generated.\\
\bottomrule
\end{tabular}
\caption{Constraint regimes induced by different interpretations of the
evidence state.  The regimes differ not in whether a KG is used, but in what the
system is allowed to infer from observed support and observed absence.}
\label{tab:design-map}
\end{table}

\subsection{Evidence construction and trajectory support}

GraphRAG, KGQA, and graph-agent systems often begin from the same operational
move: construct a query-dependent evidence state by retrieval, linking,
extraction, graph traversal, or tool calls.  The theory treats this state as an
object to be declared rather than an invisible preprocessing artifact.  Its
entity coverage, relation vocabulary, temporal scope, provenance weights, and
path-construction rules determine the meaning of every downstream support
claim.

This view is especially important for multi-hop reasoning.  A retrieved
community or neighborhood may be topically relevant while still failing to
support a particular directed transition.  Conversely, a path can be missing
because the graph snapshot is incomplete, not because the corresponding fact is
false.  The formal distinction between relevance, typed residual compatibility,
path energy, and support region gives system designers a vocabulary for these
cases.  It also explains why simply adding graph context to a prompt is not the
same as grounding a reasoning trajectory in that graph.

\subsection{Constraint strength and communicative commitment}

Different applications require different levels of constraint.  In exploratory
search or brainstorming, relevance retrieval may be sufficient: the graph helps
the model attend to useful material, while the answer remains a language-model
generation conditioned on context.  In factual QA, decision support, or
scientific summarization, soft grounding is often more appropriate because it
exposes a graded evidence signal and allows uncertainty from graph
incompleteness to remain visible.  In database-like settings with closed-world
relations, hard grounding may be justified; there the absence of an edge can
carry real negative evidence.

The communicative layer must match this constraint strength.  A system that only
uses graph retrieval can say that an answer was generated with graph context,
but it cannot honestly say that every reasoning step is graph verified.  A soft
grounding system can report evidence scores, margins, or posterior shifts.  A
hard grounding system can make stronger support claims only after declaring the
coverage assumptions under which hard rejection is valid.  Selective grounding
adds a further option: the system may answer, abstain, or explicitly label a
claim as plausible but not verified by the current evidence state.

\subsection{Implications for system design}

The first implication concerns reporting.  A KG-enhanced system benefits from
making the evidence state visible: which graph snapshot was used, which entities
were linked, which relations were traversed, which paths supported the emitted
trajectory, and which parts of the answer remained unsupported.  Such reporting
does not need to be verbose in every user interface, but the information needs
to exist at the system level if the output is presented as grounded.

A second implication is diagnostic.  Many failures attributed to ``LLM
hallucination'' are better described as mismatches among retrieval relevance,
trajectory support, and communication.  The model may have retrieved a relevant
neighborhood but reversed a relation; it may have followed a low-confidence edge;
or it may have produced a true answer whose support was absent from the
retrieved snapshot.  These cases require different fixes: typed relation
energies, reliability calibration, coverage expansion, abstention, or a more
careful output label.

A third implication concerns calibration.  The parameters in the geometric and
variational formulation are not ornamental.  Relation residual scales determine
how strongly the system trusts typed graph structure; reliability weights encode
provenance and freshness; slack costs express the penalty for missing support;
and the grounding-strength parameter $\lambda$ controls how far evidence may
move the model prior.  Treating these quantities as tunable calibration objects
is more faithful than choosing a single universal notion of ``KG grounding.''
It also makes empirical evaluation more interpretable, because improvements can
be traced to coverage, alignment, relation modeling, posterior deformation, or
abstention behavior.

This taxonomy is not intended to rank all regimes from weak to strong in a
single order.  Hard constraints are stronger only when their closure assumptions
are warranted; otherwise they can be less faithful than a softer method that
admits uncertainty.  The central design question is therefore not whether a
system uses a KG, but what kind of evidential commitment the system is entitled
to make from the KG it actually observed.

\section{Limitations and Scope}
\label{sec:limitations}

The analysis in this paper is deliberately evidence-relative.  It evaluates
compatibility with a declared observation state $R$, not truth in the latent
world.  A supported claim may be false when the retrieved graph is noisy, and an
unsupported claim may be true when the relevant edge is missing, stale, outside
the schema, or not retrieved.  The latent object $\Gstar$ is introduced only to
state identifiability questions; it is not assumed to be available to the
grounding algorithm or uniquely representable.  The same caution applies to
observable trajectories.  The paper studies text, claims, paths, and actions that
can be inspected or reconstructed; it does not claim that a written chain of
thought reveals the model's hidden causal computation.

The formal results also depend on modeling choices that are not settled by the
theory itself.  Candidate generation remains an upstream bottleneck because
KL-regularized deformation cannot assign mass to trajectories that the prior
omits.  The semantic map, entity linker, relation maps, alignment family,
reliability weights, path prior, and slack costs are all part of the declared
energy model.  Poor calibration can make paraphrases appear distant or
contradictions appear close.  Likewise, the finite-candidate setting makes the
posterior, limiting, and stability statements transparent; extensions to
continuous trajectory spaces require reference measures, integrability of
$e^{-\lambda E_R}$, and suitable compactness or tightness assumptions.  The
concentration result in Appendix~\ref{app:hoeffding} further assumes independent
evidence states, bounded energies, and a positive expected score margin; strongly
correlated retrieval failures would require different tools.

Finally, the paper does not make an empirical dominance claim.  Soft grounding is
not asserted to always improve answer accuracy, and hard grounding is not
asserted to be intrinsically harmful.  Hard constraints can be well founded for
closed-world databases, schema-complete subtasks, type systems, grammar
constraints, and safety policies.  The critique is narrower: absence from
incomplete open-world KG evidence is not a complete negative fact.  Within that
scope, the contribution is to identify the non-identifiability of unsupported
trajectories, characterize finite-energy posterior deformation, and state the
stability and calibration conditions under which KG evidence can be used without
overclaiming what it establishes.

\section{Conclusion}
\label{sec:conclusion}

We have developed a theoretical perspective on KG-enhanced LLM grounding under
incomplete evidence.  The unit of analysis is an observable or reconstructed
trajectory; the epistemic state is the retrieved, linked, temporally scoped
evidence $R$ rather than an inaccessible complete graph.  KG-induced geometry
constructs entity anchors, typed residuals, path energies, alternative-path
distributions, and support regions that measure compatibility with that state.

The central results clarify the hard--soft distinction.  A two-world argument
shows that no rule using only an incomplete observed state can distinguish every
false unsupported trajectory from every true-but-unobserved one.  Soft grounding
is the unique KL-regularized deformation of the LLM trajectory prior; finite slack
preserves candidacy, while hard conditioning arises as an infinite-penalty limit
that changes posterior support.  Bounded posterior movement can be traced to
bounded perturbations in entity alignment, typed relation residuals, reliability
weights, and finite slack costs.  Selective risk requires calibrated confidence
rather than energy thresholds alone.  Typed relation regularization further shows
why ordinary homogeneous graph smoothing is insufficient for directed KG
semantics.

The position is therefore more limited than a universal recipe for grounding.
Under incomplete open-world evidence, finite, typed, coverage-aware posterior
deformation is the coherent default.  Hard epistemic grounding remains
appropriate when its closed-world, coverage, schema, or safety assumptions are
explicit.  In both cases, KG support manages the relation between an LLM prior
and the evidence actually available to the system; it does not by itself become
factual truth.

\bibliography{reference_cited}
\bibliographystyle{iclr2026_conference}

\appendix

\section{Coverage-Limited Hard Support and Finite-Slack Preservation}
\label{app:finite-slack-bound}

This and the following appendices provide technical supplements to the main
perspective.  They do not replace the KL-regularized posterior account in
Section~\ref{sec:variational}.  They do not change the main thesis:
evidence-relative grounding is analyzed through incomplete observation states,
hard-grounding non-identifiability, KL-regularized posterior deformation, and
finite-energy stability.

\begin{proposition}[Coverage-limited hard support and finite-slack preservation]
\label{prop:finite-slack-lower-bound}
Let $\rho^\dagger$ be a candidate trajectory with
$p_0=p_\theta(\rho^\dagger\mid q)>0$.  Let $C_R$ denote the event that the
observed evidence state contains sufficient support for $\rho^\dagger$, and let
$c=\Pr_R(C_R)$.

Suppose a hard path-complete rule can retain $\rho^\dagger$ only when $C_R$
occurs.  Then
\[
 \Pr_R[\rho^\dagger\in \Feas(q;\epsilon)]\le c.
\]
Equivalently, hard retention is coverage-limited.

Now suppose the soft energy is finite on $\rho^\dagger$ and satisfies
\[
 E_R(\rho^\dagger)\le B<\infty .
\]
If $E_R(\rho)\ge0$ for all candidates, then for every finite $\lambda\ge0$,
\[
 p_\lambda^R(\rho^\dagger\mid q)
 \ge
 p_0 e^{-\lambda B}.
\]
\end{proposition}

\begin{proof}
The hard retention event is a subset of $C_R$, giving the coverage bound.  For
the soft bound, nonnegativity of $E_R$ implies $e^{-\lambda E_R(\rho)}\le1$, so
the normalizer is at most one.  The numerator for $\rho^\dagger$ is at least
$p_0e^{-\lambda B}$, proving the lower bound.
\end{proof}

\paragraph{Expected hard-mass variant.}
The same coverage limitation can be stated at the posterior-mass level.  Let
$Z_H(R)=p_\theta(\Feas(q;\epsilon)\mid q)$, and suppose the hard posterior is
evaluated only when $Z_H(R)>0$.  If retaining $\rho^\dagger$ implies the
coverage event $C_R$, then
\[
 \mathbb E_R\!\left[
 p_H^R(\rho^\dagger\mid q)\mathbf 1\{Z_H(R)>0\}
 \right]
 \le
 \Pr_R(C_R)=c.
\]
This bound is deliberately one-sided.  It says that hard posterior mass for a
correct but coverage-dependent trajectory cannot appear on evidence states in
which the required support is absent.  It does not assert a positive lower bound
without additional assumptions on $Z_H(R)$ and the prior mass of feasible
competitors.

\paragraph{Interpretation.}
The hard part isolates coverage as a first-order bottleneck.  The soft part
shows that finite slack prevents positive-prior trajectories with bounded
missing-support cost from being deleted from posterior support.  The lower bound
is not a correctness guarantee and weakens as $\lambda$ or $B$ grows.

\section{A Hoeffding-Style Bound for Averaged Evidence States}
\label{app:hoeffding}

\begin{proposition}[Averaged-evidence soft selection]
\label{prop:hoeffding-averaged-evidence}
Let $\Cand$ be finite, with nonempty
$\Cand^+\subseteq\Cand$ containing correct answer-bearing trajectories and
$\Cand^-=\Cand\setminus\Cand^+$ nonempty.  The partition
$(\Cand^+,\Cand^-)$ is defined by an external correctness criterion and is not
induced by KG compatibility.  Let
$R_1,\ldots,R_m$ be independent evidence states and assume
$E_{R_i}(\rho)\in[0,B]$ almost surely for every $i$ and $\rho$.  Define
\[
 \widehat E_m(\rho)=\frac1m\sum_{i=1}^mE_{R_i}(\rho),
 \qquad
 \widehat S_m(\rho)=\mathcal L_\theta(\rho;q)
 +\lambda\widehat E_m(\rho),
\]
where $\mathcal L_\theta$ is a deterministic prior cost and $\lambda>0$.  Let
$\bar S(\rho)=\mathbb E[\widehat S_m(\rho)]$ and suppose
\begin{equation}
 \gamma=
 \min_{\rho^-\in\Cand^-}\bar S(\rho^-)
 -\min_{\rho^+\in\Cand^+}\bar S(\rho^+)>0.
 \label{eq:expected-score-margin}
\end{equation}
Then every empirical minimizer
$\widehat\rho_m\in\arg\min_{\rho\in\Cand}\widehat S_m(\rho)$ belongs to
$\Cand^+$ with probability at least
\begin{equation}
 \Pr(\widehat\rho_m\in\Cand^+)
 \ge1-2|\Cand|\exp\!\left(
 -\frac{m\gamma^2}{2\lambda^2 B^2}\right).
 \label{eq:hoeffding-selection-bound}
\end{equation}
Consequently, it is sufficient that
\[
 m \ge
 \frac{2\lambda^2 B^2}{\gamma^2}
 \log\frac{2|\Cand|}{\delta}
\]
to obtain probability at least $1-\delta$.
\end{proposition}

\begin{proof}
For a fixed $\rho$, Hoeffding's inequality gives
\[
 \Pr\!\left(
 |\widehat S_m(\rho)-\bar S(\rho)|\ge\frac\gamma2
 \right)
 \le2\exp\!\left(-\frac{m\gamma^2}{2\lambda^2 B^2}\right),
\]
because only $\lambda\widehat E_m(\rho)$ is random.  A union bound over
$\Cand$ shows that, outside an event with probability at most the error term in
Eq.~(\ref{eq:hoeffding-selection-bound}), every score deviation is strictly less
than $\gamma/2$.  On this event, the empirically best trajectory in $\Cand^+$ has
strictly smaller score than every trajectory in $\Cand^-$ by
Eq.~(\ref{eq:expected-score-margin}).  Hence every empirical minimizer lies in
$\Cand^+$.  Solving the tail bound for $m$ yields the stated sufficient condition.
\end{proof}

\paragraph{Interpretation.}
The result is relevant to repeated retrieval, multiple KG snapshots, or graph
agents that aggregate several evidence-acquisition rounds.  It is not a general
answer-accuracy theorem: it assumes independent evidence states, bounded energies,
and a positive expected score margin relative to a correctness partition defined
outside the KG energy.  Dependence across trajectories within one evidence state
is allowed, but correlated evidence states require martingale, mixing, or other
appropriate concentration tools.

\section{Constraint-Strength Tradeoff}
\label{app:constraint-tradeoff}

\begin{proposition}[Constraint-strength tradeoff]
\label{prop:constraint-tradeoff}
Let
\[
 X_\lambda\in\arg\min_X\{F(X)+\lambda G(X)\},
\]
where $G(X)\ge0$.  Let $0\le\lambda_1<\lambda_2$ and assume minimizers
$X_{\lambda_1},X_{\lambda_2}$ exist and all displayed objective values are
finite.  Then
\[
 G(X_{\lambda_2})\le G(X_{\lambda_1}),
 \qquad
 F(X_{\lambda_2})\ge F(X_{\lambda_1}).
\]
\end{proposition}

\begin{proof}
Optimality at $\lambda_1$ and $\lambda_2$ gives
\begin{align*}
 F(X_{\lambda_1})+\lambda_1G(X_{\lambda_1})
 &\le F(X_{\lambda_2})+\lambda_1G(X_{\lambda_2}),\\
 F(X_{\lambda_2})+\lambda_2G(X_{\lambda_2})
 &\le F(X_{\lambda_1})+\lambda_2G(X_{\lambda_1}).
\end{align*}
Adding the inequalities gives
$(\lambda_2-\lambda_1)
[G(X_{\lambda_2})-G(X_{\lambda_1})]\le0$, proving the first claim.
Substitution into the first optimality inequality proves the second.
\end{proof}

\paragraph{Interpretation.}
When $F$ is an LLM-prior or minimal-intervention cost and $G$ is KG violation,
larger $\lambda$ reduces measured graph violation but increases the cost of
departing from the prior preference.  Larger grounding strength is therefore not
monotonically better.  Coverage-aware calibration is necessary because a large
$\lambda$ can produce over-grounding when $R$ is incomplete, stale, or noisy.

\section{Unified Objective and Projection-Style Grounding}
\label{app:projection}

The main paper treats grounding as posterior deformation over a finite candidate
set.  The following complementary template is useful when the system revises a
trajectory or semantic state directly.  It is not an alternative main theory and
does not assert that existing systems optimize all terms.

Let $\mathcal L_\theta(\rho;q)$ be an LLM prior cost, let
$\Omega_R(X_\rho)$ be a relation-aware representation penalty, and suppose the
displayed quantities and minimizers are well defined.  A unified trajectory
objective is
\begin{equation}
 \rho^\star\in\arg\min_{\rho\in\Tspace(q)}
 \left\{
 \mathcal L_\theta(\rho;q)
 +\lambda E_R(X_\rho;q)
 +\beta\dist^2(X_\rho,\SupportRegion(q;\epsilon))
 +\mu\Omega_R(X_\rho)
 \right\}.
 \label{eq:appendix-unified-objective}
\end{equation}
The four terms represent, respectively, LLM plausibility, graded
evidence-relative support, displacement toward an evidence support region, and
typed representation-level regularization.  The support region remains relative
to $R$ and is not a truth set.

Given an initial semantic trajectory $X^0$, a projection-style revision operator
can be written
\begin{equation}
 P_{\lambda,\mu}^R(X^0)
 \in\arg\min_X
 \left\{
 \frac12d_{\mathcal A}(X,X^0)^2
 +\lambda E_R(X;q)+\mu\Omega_R(X)
 \right\}.
 \label{eq:soft-projection-operator}
\end{equation}
When one deliberately imposes a hard support set, the corresponding metric
projection is
\begin{equation}
 \operatorname{Proj}_{\SupportRegion(q;\epsilon)}(X^0)
 \in\arg\min_{X\in\SupportRegion(q;\epsilon)}
 d_{\mathcal A}(X,X^0)^2.
 \label{eq:hard-support-projection}
\end{equation}
Existence follows under the compactness conditions of
Proposition~\ref{prop:support-region}; uniqueness requires additional geometry,
such as a nonempty closed convex surrogate in a Hilbert space.

\paragraph{Interpretation.}
Projection is appropriate when a system revises one generated trajectory rather
than reweights a finite candidate distribution.  Equations
(\ref{eq:appendix-unified-objective})--(\ref{eq:hard-support-projection}) provide
an interpretation of entity correction, relation-aware revision, and movement
toward a declared support region.  Hard projection inherits the coverage and
completeness requirements analyzed in Section~\ref{sec:hard}.  No algorithmic or
empirical claim is made here.

\section{Set-Level Energy-Gap Suppression}
\label{app:set-gap}

\begin{proposition}[Set-level energy-gap suppression]
\label{prop:set-energy-gap}
Fix $R$ and let nonempty $B,S\subseteq\Cand(q)$ have positive prior mass.  Suppose
there is $\Delta>0$ such that
\[
 E_R(\rho_b)\ge E_R(\rho_s)+\Delta
 \qquad
 \text{for every }\rho_b\in B,\ \rho_s\in S.
\]
Then
\[
 \frac{p_\lambda^R(B\mid q)}{p_\lambda^R(S\mid q)}
 \le
 \frac{p_\theta(B\mid q)}{p_\theta(S\mid q)}e^{-\lambda\Delta}.
\]
\end{proposition}

\begin{proof}
Let $m_B=\min_{\rho_b\in B}E_R(\rho_b)$ and
$M_S=\max_{\rho_s\in S}E_R(\rho_s)$.  Finiteness of $\Cand$ and the assumed gap
give $m_B\ge M_S+\Delta$.  Hence
\[
 p_\lambda^R(B\mid q)
 \le\frac{e^{-\lambda m_B}p_\theta(B\mid q)}{Z_R},
 \qquad
 p_\lambda^R(S\mid q)
 \ge\frac{e^{-\lambda M_S}p_\theta(S\mid q)}{Z_R}.
\]
Dividing and using the gap proves the result.
\end{proof}

\paragraph{Interpretation.}
When reliable counter-support creates a uniform energy gap, a set of contradicted
trajectories is exponentially suppressed relative to a supported set.  The claim
does not apply to merely unsupported trajectories.  Contradiction dominance
requires explicit counter-support or a calibrated contradiction energy; absence
from $R$ alone does not provide the premise.

\section{Hard Conditioning under Complete Observed Support}
\label{app:hard-baseline}

\begin{proposition}[Coverage-conditional amplification]
\label{prop:hard-amplification}
Let $F_R\subseteq\Cand(q)$ be a hard feasible set with
$Z_R^H=p_\theta(F_R\mid q)>0$, and suppose the set of correct trajectories
$\Cand^+$ satisfies $\Cand^+\subseteq F_R$.  Define
\[
 p_H^R(\rho\mid q)=
 \frac{p_\theta(\rho\mid q)\mathbf1\{\rho\in F_R\}}{Z_R^H}.
\]
Then
\[
 p_H^R(\Cand^+\mid q)
 =\frac{p_\theta(\Cand^+\mid q)}{Z_R^H}
 \ge p_\theta(\Cand^+\mid q).
\]
The inequality is strict if positive prior mass is removed outside $F_R$ and
$p_\theta(\Cand^+\mid q)>0$.
\end{proposition}

\begin{proof}
Because $0<Z_R^H\le1$ and $\Cand^+\subseteq F_R$, conditioning divides the total
prior mass of $\Cand^+$ by a quantity no greater than one.  Strictness follows
when $Z_R^H<1$ and the numerator is positive.
\end{proof}

\paragraph{Interpretation.}
Hard conditioning can amplify correct feasible trajectories when observed
coverage is complete for them.  This baseline does not contradict
Theorem~\ref{thm:hard-impossibility}: the theorem concerns worlds in which a
correct trajectory may be true but unobserved.  The risk begins when the premise
$\Cand^+\subseteq F_R$ cannot be guaranteed.

\section{Spectral View of Classical Laplacian Smoothing}
\label{app:spectral-laplacian}

\begin{proposition}[Spectral decomposition of Laplacian energy]
\label{prop:spectral-laplacian}
For a homogeneous undirected graph, let
$L=U\Lambda U^\top$ with eigenvalues
$0=\lambda_1\le\cdots\le\lambda_n$, and let
$\widehat H=U^\top H$.  Then
\[
 \operatorname{Tr}(H^\top L H)
 =\sum_{\ell=1}^n\lambda_\ell
 \|\widehat h_\ell\|_2^2,
\]
where $\widehat h_\ell$ is the $\ell$th row of $\widehat H$.
\end{proposition}

\begin{proof}
Substitute $L=U\Lambda U^\top$ and use cyclic invariance of the trace:
\[
 \operatorname{Tr}(H^\top L H)
 =\operatorname{Tr}(\widehat H^\top\Lambda\widehat H)
 =\sum_{\ell=1}^n\lambda_\ell\|\widehat h_\ell\|_2^2.
\]
\end{proof}

\paragraph{Interpretation.}
Classical Laplacian regularization suppresses high-frequency variation on a
homogeneous undirected graph, which is appropriate under a homophily assumption.
Typed directed KGs generally require the relation-aware residual in
Eq.~(\ref{eq:typed-regularizer}) instead.  This spectral observation supports the
special-case comparison in Section~\ref{sec:regularization}; it is not central to
the posterior theory.

\end{document}